\documentclass[preprint,3p,times]{elsarticle}

\usepackage{pifont}
\usepackage{mathrsfs}
\usepackage{multirow}
\usepackage{amssymb,amsmath,amsfonts}
\usepackage{graphicx,algorithm,algorithmic,epsfig,subfigure}
\usepackage{longtable}
\usepackage{graphpap}
\usepackage{caption}
\usepackage{comment}
\usepackage{multirow}
\usepackage{threeparttable}
\usepackage{array}
\usepackage{url}
\usepackage{color}
\usepackage[rotateright]{rotating}

\newtheorem{lemma}{Lemma}
\newtheorem{theorem}{Theorem}
\newtheorem{definition}{Definition}

\begin{document}

\begin{frontmatter}

\title{A binary differential evolution algorithm learning from explored solutions}

\author[WHUT,WHU]{Yu Chen\corref{cor1}}
\ead{chymath@gmail.com}
\author[WHU]{Weicheng Xie}
\author[WHU]{Xiufen Zou}

\address[WHUT]{School of Science, Wuhan University of Technology, Wuhan, 430070, China}
\address[WHU]{School of Mathematics and Statistics, Wuhan University, Wuhan, 430072, China}

\cortext[cor1]{Corresponding author. Tel:86-27-87108027}

\begin{abstract}
  Although real-coded differential evolution (DE) algorithms can perform well on continuous optimization problems (CoOPs), it is still a challenging task to design an efficient binary-coded DE algorithm. Inspired by the learning mechanism of particle swarm optimization (PSO) algorithms, we propose a binary learning differential evolution (BLDE) algorithm that can efficiently locate the global optimal solutions by learning from the last population. Then, we theoretically prove the global convergence of BLDE, and compare it with some existing binary-coded evolutionary algorithms (EAs) via numerical experiments. Numerical results show that BLDE is competitive to the compared EAs, and meanwhile, further study is performed via the change curves of a renewal metric and a refinement metric to investigate why BLDE cannot outperform some compared EAs for several selected benchmark problems. Finally, we employ BLDE solving the unit commitment problem (UCP) in power systems to show its applicability in practical problems.

\end{abstract}


\end{frontmatter}

\section{Introduction}
\label{sec_1}
\subsection{Background}
Differential evolution (DE) \cite{Storn1997}, a competitive evolutionary algorithm emerging more than a decade ago, has been widely utilized in the science and engineering fields \cite{Price2005,Das2011}. The simple and straightforward evolving mechanisms of DE endow it with powerful capability of solving continuous optimization problems (CoOPs), however, hamper its applications on discrete optimization problems (DOPs).

To take full advantage of the superiority of mutations in classic DE algorithms, Pampar\'{a} and Engelbrecht \cite{Pampara2006} introduced a trigonometric generating function to transform the real-coded individuals of DE into binary strings, and proposed an angle modulated differential evolution (AMDE) algorithm for DOPs. Compared with the binary differential evolution (BDE) algorithms that directly manipulate binary strings,  AMDE was much slower but outperformed  BDE algorithmsx  with respect to accuracy of the obtained solutions \cite{Engelbrecht2007}. Meanwhile, Gong and Tuson proposed a binary DE algorithm by forma analysis \cite{Gong2007}, but it cannot perform well on binary constraint satisfaction problems due to its weak exploration ability \cite{Yang2008}. Trying to simulate the operation mode of the continuous DE mutation, Kashan {\em et al.} \cite{Kashan2013} design a dissimilarity based differential evolution (DisDE) algorithm incorporating a measure of dissimilarity in mutation. Numerical results show that DisDE is competitive to some existing binary-coded evolutionary algorithms (EAs).

Moreover, the performances of BDE algorithms can also be improved by incorporating recombination operators of other EAs. Hota and Pat \cite{Hota2010} proposed an adaptive quantum-inspired differential evolution algorithm (AQDE) applying quantum computing techniques, while He and Han \cite{He2007} introduced the negative selection in artificial immune systems to obtain an artificial immune system based differential evolution (AIS-DE) algorithm.  With respect to the fact that the logical operations introduced in AIS-DE tends to produce ``1'' bits with increasing probability, Wu and Tseng \cite{Wu2010} proposed an modified binary differential evolution strategy to improve the performance of BDE algorithms on topology optimization of structures.

\subsection{Motivation and Contribution}
Existing researches tried to incorporate the recombination strategies of various EAs to get efficient BDEs for DOPs, whereas there are still some points to be improved:
\begin{itemize}
\item AMDE \cite{Pampara2006} has to transform real values to binary strings, which leads to the explosion of computation cost for function evaluations. Meanwhile, the mathematical properties of the transformation function can also influence its performances on various DOPs;
\item BDE algorithms directly manipulating bit-strings, such as binDE \cite{Gong2007}, AIS-DE \cite{He2007} and MBDE \cite{Wu2010}, etc., cannot effectively imitate the mutation mechanism of continuous DE algorithms. Thus, they cannot perform well on high-dimensional DOPs due to their weak exploration abilities;
\item DisDE \cite{Kashan2013}, which incorporates a dissimilarity metric in the mutation operator, has to solve a minimization problem during the mutation process. As a consequence, the computation complexity of DisDE is considerably high.
\end{itemize}

Generally, it is a challenging task to design an efficient BDE algorithm perfectly addressing the aforementioned points. Recently, variants of the particle swarm optimization (PSO) algorithm \cite{Kennedy1995} have been successfully utilized in real applications \cite{Eberhart2001,Banks2007,Poli2007,Banks2008,Kennedy2010}. Although DE algorithms perform better than PSO algorithms in some real world applications \cite{Vesterstrom2004,Rekanos2008,Ponsich2011}, it is still promising to improve DE by incorporating PSO in the evolving process \cite{Das2005,Moore2006,Omran2009}. Considering that the learning mechanism of PSO can accelerate the convergence of populations, we propose a hybrid binary-coded evolutionary algorithm learning from the last population, named as the binary learning differential evolution (BLDE) algorithm. In BLDE, the searching process of population is guided by the renewed information of individuals, the dissimilarity between individuals and the best explored solution in the population. By this means, BLDE can performance well on DOPs.

The remainder of the paper is structured as follows. Section 2 presents a description of BLDE, and its global convergence is theoretically proved in Section 3. Then, in Section 4 BLDE is compared with some existing algorithms by numerical results. To  test performance of BLDE on real-life problems, we employ it to solve the unit commitment problem (UCP) in Section 5. Finally, discussions and conclusions are presented in Section 6.

\section{The binary learning differential evolution algorithm}

\subsection{Framework of the binary learning differential evolution algorithm}

\begin{algorithm}[htb!]
\caption{The binary learning differential evolution (BLDE) algorithm}  \label{BLDE}
\begin{algorithmic}[1]
\STATE Randomly generate two populations $\mathbf X^{(1)}$ and $\mathbf A^{(1)}$ of $\mu$ individuals; Set $t:=1$;
\WHILE{the stop criterion is no satisfied}
\STATE Let $\mathbf x_{gb}=(x_{gb,1},\dots,x_{gb,n})\triangleq\arg\max\limits_{\mathbf x\in \mathbf X^{(t)}}\{f(\mathbf x)\}$;
\FORALL{$\mathbf w\in\mathbf X^{(t)}$}
\STATE Randomly select $\mathbf x=(x_1,\dots,x_n)$ and $\mathbf y=(y_1,\dots,y_n)$ from $\mathbf X^{(t)}$, as well as $\mathbf z=(z_1,\dots,z_n)$ from $\mathbf A^{(t)}$;
\STATE $\mathbf {tx}=(tx_{1},\dots,tx_{n})\triangleq\arg\max\{f(\mathbf y),f(\mathbf z)\}$;
\FOR{$j=1,2, \cdots, n$ }
\IF{$y_j= z_j$}
  \IF{$x_{gb,j}\neq x_{j}$}
    \STATE $tx_{j}= x_{gb,j}$;
  \ELSE
     \IF{$rand(0,1)\leq p$}
       \STATE
             $tx_{j}=\left\{ \begin{aligned}&0&\quad &\mbox{with probability } \frac{1}{2};\\&1&\quad &\mbox{otherwise.}\\ \end{aligned}\right. $
     \ENDIF
  \ENDIF
\ENDIF
\ENDFOR
\IF {$f(\mathbf {tx})\ge f(\mathbf w)$}
 \STATE  $\mathbf w=\mathbf {tx}$;
\ENDIF
\ENDFOR
\STATE $t:=t+1$;
\STATE $\mathbf A^{(t)}=\mathbf X^{(t-1)}$;
\ENDWHILE
\end{algorithmic}
\end{algorithm}

For a binary maximization problem (BOP) \footnote{When a CoOP is considered, the real-value variables can be coded as bit-strings, and consequently, a binary optimization problem is constructed to be solved by binary-coded evolutionary algorithms.}
 \begin{equation}\label{OP}
 \max_{\mathbf x\in S}\quad f(\mathbf x)=f(x_1,\dots,x_n),\quad S\subset\{0,1\}^n,
 \end{equation}
 the BLDE algorithm illustrated by Algorithm \ref{BLDE} possesses two collections of $\mu$ solutions, the population $\mathbf X^{(t)}$ and the archive $\mathbf A^{(t)}$. At the first generation, the population $\mathbf X^{(1)}$ and the archive $\mathbf A^{(1)}$  are generated randomly. Then, repeat the following operations until the stopping criterion is satisfied.

 For each individual $\mathbf w\in\mathbf X^{(t)}$ a trial solution is generated by three randomly selected individuals $\mathbf x,\,\,\mathbf y\in \mathbf X^{(t)}$ and $\mathbf z\in \mathbf A^{(t)}$. At first, initialize the trial individual $\mathbf {tx}=\{tx_1,\dots,tx_n\}$ as the winner of two individuals $\mathbf y\in \mathbf X^{(t)}$ and $\mathbf z\in \mathbf A^{(t)}$. $\forall\,j\in\{1,2,\dots,n\}$, if $\mathbf y$ and $\mathbf z$ coincide on the $j^{th}$ bit, the $j^{th}$ bit of $\mathbf {tx}$ is changed as follows.
\begin{itemize}
\item If the $j^{th}$ bit of $\mathbf x$ differs from that of $\mathbf x_{gb}$, $tx_j$ is set to be $x_{gb,j}$, the $j^{th}$ bit of $\mathbf x_{gb}$;
\item otherwise, $tx_j$ is randomly mutated with a preset probability $p$.
\end{itemize}
Then, replace $\mathbf w$ with $\mathbf{tx}$ if $f(\mathbf{tx})\ge f(\mathbf w)$. After the update of population $\mathbf X^{(t)}$ is completed, set $t=t+1$ and $\mathbf A^{(t)}=\mathbf X^{(t-1)}$.

\subsection{The positive functions of the learning scheme}
Generally speaking, the trial solution $\mathbf{tx}$ is generated by three randomly selected individuals. Meanwhile, it also incorporates conditional learning strategies in the mutation process.
\begin{itemize}
\item By randomly selecting $\mathbf y\in\mathbf X^{(t)}$, BLDE can learn from any member in the present population. Because the elitism strategy is employed in the BLDE algorithm, BLDE could learn from any {\em pbest} solution in the population, unlike that particles in PSO can only learn from their own {\em pbest} individuals.
\item By randomly selecting $\mathbf z\in\mathbf A^{(t)}$, BLDE can learn from any member in the last population. At the early stage of the iteration process, individuals in the population $\mathbf X^{(t)}$ are usually different with those in $\mathbf A^{(t)}=\mathbf X^{(t-1)}$. Combined with the first strategy, this scheme actually enhances the exploration ability of the population, and to some extent, accelerates convergence of the population.
\item When bits of $\mathbf y$ coincide with the corresponding bits of $\mathbf z$, trial solutions learn from the {\em gbest} on condition that randomly selected $\mathbf x\in\mathbf X^{(t)}$ differs from $\mathbf x_{gb}$ on the these bits. This scheme imitates the learning strategy of PSO, and meanwhile, can also prevent the population from being governed by dominating patterns, because the increase of probability $\mathbf P\{x_{gb,j}=x_j\}$ will lead to the random mutation performed on $\mathbf {tx}$, preventing the duplicate of the dominating patterns in the population.
\end{itemize}

In PSO algorithms, each particle learns from the {\em pbest} (the best solution it has obtained so far) and the {\em gbest} (the best solution the swarm has obtained so far), and particles in the swarm only exchange information via the {\em gbest} solution. The simple and unconditional learning strategy of PSO  usually results in its fast convergence rate, however, sometimes leads to its premature convergence to local optima. The BLDE algorithm learning from $\mathbf X^{(t)}$ as well as $\mathbf A^{(t)}$ can explore the feasible region in a better way, and meanwhile, by conditionally learning from $\mathbf x_{gb}$ it will not be attracted by local optimal solutions.

\section{Convergence analysis of BLDE}
\label{subsec_2}
Denote $\mathbf x^*$ to be an optimal solution of BOP (\ref{OP}),
the global convergence of BLDE can be defined as follows.
\begin{definition}
\label{def_2}
Let $\{\mathbf X^{(t)},t=1,2,dots\}$ be the population sequence of BLDE. It is said to converge in probability to the optimal solution $\mathbf x^*$ of BOP (\ref{OP}), if it holds that
 $$\lim_{t\rightarrow \infty}\mathbf P\{\mathbf x^*\in \mathbf X^{(t)}\}=1.$$
\end{definition}
To confirm the global convergence of the proposed BLDE algorithm, we first show that any feasible solution can be generated with a positive probability.
\begin{lemma}
\label{lem_1}
In two generations, BLDE can generate any feasible solution of BOP (\ref{OP}) with a probability greater than or equal to a positive constant $c$.
\end{lemma}
\begin{proof}
Denote $\mathbf x^{(t)}(i)=(x^{(t)}_1(i),\dots,x^{(t)}_n(i))$ and $\mathbf a^{(t)}(i)=(a^{(t)}_1(i),\dots,a^{(t)}_n(i))$ to be the $i^{th}$ individuals of $\mathbf X^{(t)}$ and $\mathbf A^{(t)}$, respectively. Let $\mathbf {tx}^{(t)}(i)=({tx}_1^{(t)}(i),\dots,{tx}_n^{(t)}(i))$ be the $i^{th}$ trial individual generated at the $t^{th}$ generation. There are two different cases to be investigated.
\begin{enumerate}
\item If $\mathbf X^{(t)}$ and $\mathbf A^{(t)}$ include at least one common individual, the probability $\mathbf P\{\mathbf y=\mathbf z\}$ is greater than or equal to $\frac{1}{\mu^2}$, where $\mathbf y\in\mathbf X^{(t)}$ and $\mathbf z\in\mathbf A^{(t)}$ are selected randomly from $\mathbf X^{(t)}$ and $\mathbf A^{(t)}$, respectively. Then, the random mutation illustrated by Lines 12 - 14 of Algorithm 1 will be activated with probability $\frac{1}{\mu}$, which is the minimum probability of selecting $\mathbf x$ to be $\mathbf x^{(t)}_{gb}$, the best individual in the present population $\mathbf X^{(t)}$. For this case, both $\mathbf P\{tx_j=0\}$ and $\mathbf P\{tx_j=1\}$ are greater than or equal to $\frac{p}{2\mu^3}$. Then, any feasible solution can be generated with a positive probability greater than or equal to $\left(\frac{p}{2\mu^3}\right)^n$.
\item If all individuals in $\mathbf X^{(t)}$ differ from those in $\mathbf A^{(t)}$, two different solutions $\mathbf y\in\mathbf X^{(t)}$ and $\mathbf z\in\mathbf A^{(t)}$ are located at the same index $i_0$ with probability
     $$\mathbf P\{\mathbf y=\mathbf x^{(t)}(i_0),\mathbf z=\mathbf a^{(t)}(i_0)\}=\frac{1}{\mu^2}.$$

     Since $\mathbf y\neq\mathbf z$, $I_1=\{j;y_j\neq z_j\}$ is not empty. Moreover, the elitism update strategy ensure that the trial individual $\mathbf{tx}^{(t)}(i_0)$ is initialized to be $\mathbf{tx}^{(t)}(i_0)=\mathbf y$. Then, $$tx_j^{(t)}(i_0)=y_j=x_j^{(t)}(i_0),\quad \forall\,j\in I_1,$$
     and $\forall\,j\notin I_1$, $\mathbf{tx}^{(t)}(i_0)$ will keep unchanged with a probability greater than $\frac{1-p}{\mu}$, the probability of selecting $\mathbf x=\mathbf x_{gb}$ and not activating the mutation illustrated by Lines 12-14 of Algorithm 1. That is to say, the probability of generating a trial individual $\mathbf{tx}^{(t)}(i_0)=\mathbf y=\mathbf x^{(t)}(i_0)$ is greater than or equal to $\frac{1-p}{\mu^3}$.

     For this case, the $i_0^{th}$ individual of the population will keep unchanged at the $t^{th}$ generation, and at the next generation (generation $t+1$), $\mathbf x^{(t+1)}(i_0)$ will coincide with $\mathbf a^{(t+1)}(i_0)$. Then, it comes to the first case, and consequently, the trial individual $\mathbf {tx}^{(t+1)}(i)$ can reach any feasible solution with a positive probability greater than or equal to $\left(\frac{p}{2\mu^3}\right)^n$. For this case, any feasible solution can be generated with a probability greater than $\frac{1-p}{\mu^3}\left(\frac{p}{2\mu^3}\right)^n$.

\end{enumerate}
In conclusion, in two generations the trial individual $\mathbf {tx}$ will reach any feasible solution  with a probability greater than or equal to a positive constant $c$, where $c=\frac{1-p}{\mu^3}\left(\frac{p}{2\mu^3}\right)^n$.
\end{proof}

\begin{theorem}
\label{the_1}
BLDE converges in probability to the optimal solution $\mathbf x^*$ of OP (\ref{OP}).
\end{theorem}
\begin{proof}
 Lemma 1 shows that there exists a  positive number $c>0$ such that  $$\mathbf P \{\mathbf x^*\in\mathbf{X}^{(t+2)}\mid\mathbf x^*\notin\mathbf{X}^{(t)}\}\ge c,\quad \forall\,t\ge 1.$$
 Denoting $$P=\mathbf P \{\mathbf x^*\in\mathbf{X}^{(t+2)}\mid\mathbf x^*\notin\mathbf{X}^{(t)}\},$$
we know that $$\mathbf P\{\mathbf x^*\notin\mathbf X^{(t+2)}|\mathbf x^*\notin\mathbf X^{(t)}\}=1-P.$$
Thus,
\begin{equation*}
\mathbf P\{\mathbf x^*\in X^{(t)}\}=1-\mathbf P\{\mathbf x^*\notin X^{(t)}\}
=1-\mathbf P\{\mathbf x^*\notin\mathbf X^{(t)}|\mathbf x^*\notin\mathbf X^{(t-2)}\}.
\end{equation*}
If $t$ is even,
\begin{eqnarray*}
\lim_{t\rightarrow \infty}\mathbf P\{\mathbf x^*\in X^{(t)}\}&=&1-\lim_{t\rightarrow \infty}\mathbf P\{\mathbf x^*\notin X^{(t)}\}\\
&=&1-\lim_{t\rightarrow \infty}(1-p)^{t/2}\mathbf P\{\mathbf x^*\notin\mathbf X^{(0)}\}\\
&=&1;
\end{eqnarray*}
otherwise,
\begin{eqnarray*}
\lim_{t\rightarrow \infty}\mathbf P\{\mathbf x^*\in X^{(t)}\}&=&1-\lim_{t\rightarrow \infty}\mathbf P\{\mathbf x^*\notin X^{(t)}\}\\
&=&1-\lim_{t\rightarrow \infty}(1-p)^{(t-1)/2}\mathbf P\{\mathbf x^*\notin\mathbf X^{(1)}\}\\
&=&1.
\end{eqnarray*}
In conclusion, BLDE converges in probability to the optimal solution $\mathbf x^*$ of BOP (\ref{OP}).
\end{proof}

\section{Numerical experiments}
Although Theorem 1 validates the global convergence of the BLDE algorithm, its convergence characteristics have not been investigated. In this section, we try to show its competitiveness to existing algorithms by numerical experiments.

\subsection{Benchmark problems }
Tab. \ref{tab_0} illustrates the selected benchmark problems, properties and settings of which are listed in Tab. \ref{tab_1}. As for the continuous problems $P_3-P_7$, all real variables are coded by bit-strings. For the multiple knapsack problem (MKP) $P_8$, we test BLDE via five test instances characterized by data files ``weing6.dat, sent02.dat, weish14.dat, weish22.dat and weish30.dat'' \cite{Website}, termed as $P_{8-1}$, $P_{8-2}$, $P_{8-3}$, $P_{8-4}$ and $P_{8-5}$, respectively. When a candidate solution is evaluated, it is penalized by $PT(\mathbf x)=\frac{1+\max_j p_i}{\min_{i,j} w_{i,j}}\cdot \max_i\{\max_j (w_{i,j}x_j-W_i),0\}$ \cite{Uyar2005}.

\begin{table}[htb!]
 \renewcommand{\arraystretch}{1.5}
 \caption{Descriptions of the selected benchmark problems.}\label{tab_0}
 \centering
 \begin{tabular}{cl}
 \hline\hline
 Problems & Descriptions.  \\
 \hline
 $P_1$: &
  $\max\quad f_{1}({\bf x})=\sum\limits_{i=1}^n\prod\limits_{j=1}^i x_j,\quad x_j\in\{0,1\},j=1,\dots, n.$\\

$P_2$: &
  $Long \,\,Path \,\,Problem: \,\,\,Root2path$ \cite{Horn1994} \\

$P_3$: &
  $\max\quad f_{3}({\bf x})=-\max\limits_{i=1,\dots,m}|x_i|,\quad x_i\in [-10,10],i=1,\dots,D. $\\

$P_4$: &
  $\max\quad f_{4}({\bf x})=-\frac{1}{4000}\sum\limits_{i=1}^D(x_i-100)^2+\prod\limits_{i=1}^D cos(\frac{x_i-100}{\sqrt{i}})-1,\quad x_i\in [-300,300],i=1,\dots,D.$\\

$P_5$: &
  $\max\quad f_{5}({\bf x})=-\sum\limits_{i=1}^D ix_i^4-rand[0,1),\quad x_i\in [-1.28,1.28],i=1,\dots,D.$\\

$P_6$: &
  $\max\quad f_{6}({\bf x})=-\sum\limits_{i=1}^{D-1}(100(x_{i+1}-x_i^2)^2+(1-x_i)^2),\quad x_i\in [-2.048,2.048],i=1,\dots,D.$\\

$P_7$: &
  $\max\quad f_{7}({\bf x})=-20+20\exp(-0.2\sqrt{\frac{1}{m}\sum\limits_{i=1}^m x_i^2})+\exp(\frac{1}{m}\sum\limits_{i=1}^m cos(2\pi x_i))-e,\quad x_i\in [-30,30],i=1,\dots,D.$\\

$P_8:$ &
  $\max \quad  f_8(\mathbf x)=\max \sum\limits_{j=1}^n p_jx_j,
  s.t. \quad  \sum\limits_{j=1}^n w_{i,j}x_j\le W_i,\quad i=1,\cdots,m,
   x_j\in \{0,1\},\quad j=1,\dots,n.$\\
 \hline
\end{tabular}
\end{table}

\begin{table}[htbp]
 \renewcommand{\arraystretch}{1.5}
 \caption{Properties and settings of the benchmark problems}\label{tab_1}
 \centering
 \begin{tabular}{cccccc}
 \hline\hline
  Problem& Binary/Real & Dimension & Bit-length& Constraints & Maximum Objective Value\\
 \hline
 $P_1$& Binary &30 &30 & -& 30\\
 $P_2$& Binary &29 &29 & -& 49992\\
$P_3$& Real &30 &180 & -& 0\\
$P_4$& Real &30 &480 & -& 0\\
$P_5$& Real &30 &240 & -& 0\\
$P_6$& Real &30 &300 & -& 0\\
$P_7$& Real &30 &300 & -& 0\\
$P_{8-1}$& Binary &28 &28& 2 & 130623\\
$P_{8-2}$& Binary &60 &60& 30 & 8722\\
$P_{8-3}$& Binary &60 &60& 5 & 6954\\
$P_{8-4}$& Binary &80 &80& 5 & 8947\\
$P_{8-5}$& Binary &90 &90& 5 & 11191\\
 \hline
\end{tabular}
\end{table}

\subsection{Parameter settings}
For numerical comparisons, BLDE is compared with the angle modulated particle swarm optimization (AMPSO) \cite{Pampara2005}, the angle modulated differential evolution (AMDE) \cite{Pampara2006}, the dissimilarity artificial bee colony (DisABC) algorithm \cite{Kashan2012}, the binary particle swarm optimization (BPSO) algorithm \cite{Kennedy1997}, the binary differential evolution (binDE) \cite{Gong2007} algorithm and the self-adaptive quantum-inspired differential evolution (AQDE) algorithm\cite{Hota2010}. As is suggested by the designers of the algorithms, the parameters of AMPSO, AMDE, DisABC, BPSO, binDE, and AQDE, are listed in Table \ref{tab_2}. Prerun for BLDE shows that when the mutation ability $p$ is less than 0.05, its weak exploration ability leads to its premature to the local optima of multi-modal problems; while when $p$ is greater than $\min\{0.15,10/n\}$, it cannot efficiently exploit the local region of global optima. Thus, in this paper we set $p=\max\{0.05\min\{0.15,10/n\}\}$ to keep a balance between exploration and exploitation. All compared algorithms are tested with a population of size 50, and the results are compared after $300\times n$ FEs, except that numerical results are compared after $300\times n\times m$ function evaluations (FEs) for MKPs, where $n$ is the bitstring length, $m$ is the number of constraints for MKP.

\begin{table*}[htbp]
\centering
\renewcommand{\arraystretch}{1.5}
\caption{Parameter settings for the tested algorithms}\label{tab_2}
\begin{tabular}{cl}
\hline\hline
Algorithm & {Parameter settings}\\
  \hline
  AMPSO & {$c_1=1.496180, c_2=1.496180, \phi=0.729844, V_{max}=4.0.$
  } \\

  AMDE &   $CR=0.25, F=1.$
   \\
  DisABC & {$\phi_{max}=0.9,\phi_{min}=0.5,p_s=0.5,N_{local}=50,p_{local}=0.01.$ } \\

  BPSO &  {$C=2, V_{max}=6.0$. } \\

  binDE &  {$F=0.8, CR=0.5$. }  \\

  AQDE &  {$F=0.1*r_1*r_2,CR=0.5+0.0375*r_3$, $r_1,r_2\thicksim U(0,1)$, $r_3\thicksim N(0,1)$. } \\

  \textbf{BLDE} &  {$p=\max(0.05,\min(0.15,10/n))$.} \\
  \hline
\end{tabular}
\end{table*}

\subsection{Numerical comparisons}\label{SN}
Implemented by the MATLAB package, the compared algorithms are run on a PC
with a INTEL(R) CORE(R) CPU, running at 2.8GHZ with 4 GB RAM.  After 50 independent runs for each problem, the results are compared in Tab. \ref{tab_3} via the average best fitness (AveFit), the standard deviation of best fitness (StdDev), the success rate (SR) and the expected runtime (RunTime). Taking AveFit and StdDev as the sorting indexes, the overall ranks of the compared algorithms are list in Tab. \ref{tab_4}.

\begin{sidewaystable}[htbp]
\footnotesize
\caption{Numerical results of AMPSO,  AMDE,DisABC and BLDE on the 12 test problems. The best results for each problem are highlighted by boldface type.}\label{tab_3}
\centering
 \begin{tabular}{cccccccc}
 \hline\hline
\multirow{3}{*}{Problem}&AMPSO&AMDE&DisABC&BPSO&binDE&AQDE&BLDE \\
\cline{2-8}
&AveFit$\pm$ StdDev&AveFit$\pm$ StdDev&AveFit$\pm$ StdDev&AveFit$\pm$ StdDev&AveFit$\pm$ StdDev&AveFit$\pm$ StdDev&AveFit$\pm$ StdDev\\
&(SR,Runtime)&(SR,Runtime)&(SR,Runtime)&(SR,Runtime)&(SR,Runtime)&(SR,Runtime)&(SR,Runtime)\\
\hline
\multirow{2}{*}{$P_{1}$}&{\bf 3.00E+01$\pm$0.00E+00}&{\bf 3.00E+01$\pm$0.00E+00}&{\bf 3.00E+01$\pm$0.00E+00}&{\bf 3.00E+01$\pm$0.00E+00}&{2.94E+01$\pm$3.14E-01}&{2.34E+01$\pm$2.88E+00}&{\bf 3.00E+01$\pm$0.00E+00}\\
&({\bf 100}, 3.01E-01)&({\bf 100}, 2.78E-01)&({\bf 100}, 1.60E+01)&({\bf 100}, 2.95E-01)&({96}, 4.07E-01)&({4}, 2.44E-01)&({\bf 100}, {\bf 2.15E-01})\\
\hline
\multirow{2}{*}{$P_{2}$}&{\bf 5.0E+04$\pm$0.00E+00}&{5.0E+04$\pm$1.54E+02}&4.53E+04$\pm$7.19E+03&{3.96E+04$\pm$1.65E+04}&{4.52E+04$\pm$8.92E+03}&{3.46E+04$\pm$1.37E+04}&{5.00E+04$\pm$6.09E+01}\\
&(\textbf{100}, 2.34E+02)&({\bf 88}, 2.03E+02)&(34, 2.92E+02)&(66, 2.81E+02)&({40}, 2.79E+02)&({16}, 2.96E+02)&({96}, 3.07E+02)\\
\hline
\multirow{2}{*}{$P_{3}$}&-8.92E+00$\pm$2.15E+00&-5.48E+00$\pm$3.21E+00&-6.88E+00$\pm$2.86E-01&-4.88E+00$\pm$7.39E-01&-6.34E+00$\pm$3.04E-01&-6.55E+00$\pm$3.68E-01&{\bf -3.22E+00$\pm$8.74E-01}\\
&(0, {\bf 3.45E+02})&({\bf 2}, 3.47E+02)&(0, 3.75E+02)&(0, 3.53E+02)&(0, 3.53E+02)&(0, 3.55E+02)&(0, 3.53E+02)\\
\hline
\multirow{2}{*}{$P_{4}$}&-4.55E+01$\pm$3.53E+01&-1.12E+01$\pm$1.99E+01&-5.70E+01$\pm$5.48E+00&-6.18E+00$\pm$2.40E+00&-4.07E+01$\pm$4.27E+00&-1.57E+01$\pm$3.79E+00&{\bf -1.12E+00$\pm$1.10E-01}\\
&(0, {\bf 1.03E+03})&({\bf 48}, 1.04E+03)&(0, 1.21E+03)&(0, 1.06E+03)&(0, 1.04E+03)&(0, 1.05E+03)&(0, 1.05E+03)\\
\hline
\multirow{2}{*}{$P_{5}$}&-1.13E+01$\pm$1.15E+01&-1.27E+00$\pm$3.67E+00&-3.11E+01$\pm$5.55E+00&{\bf -1.90E-02$\pm$8.20E-03}&-2.35E+01$\pm$3.91E+00&-2.32E+01$\pm$4.37E+00&{-5.79E-02$\pm$2.24E-02}\\
&({\bf 22}, {\bf 4.72E+02})&({\bf 22}, 4.76E+02)&(0, 5.21E+02)&(10, 4.82E+02)&(0, 4.83E+02)&(0, 4.85E+02)&(0, 4.84E+02)\\
\hline
\multirow{2}{*}{$P_{6}$}&-2.94E+03$\pm$9.26E+02&-1.18E+02$\pm$3.51E+02&-4.23E+03$\pm$4.05E+02&-5.54E+02$\pm$2.82E+02&-3.58E+03$\pm$2.92E+02&-2.02E+03$\pm$4.17E+02&{\bf -4.55E+01$\pm$9.68E+01}\\
&(0, {\bf 6.37E+02})&({\bf 8}, 6.41E+02)&(0, 7.00E+02)&(0, 6.49E+02)&(0, 6.48E+02)&(0, 6.51E+02)&(0, 6.45E+02)\\
\hline
\multirow{2}{*}{$P_{7}$}&-7.87E+00$\pm$3.29E+00&-4.57E+00$\pm$2.84E+00&-1.10E+01$\pm$3.19E-01&{\bf -1.67E+00$\pm$5.40E-03}&-1.06E+01$\pm$2.74E-01&-1.00E+01$\pm$6.43E-01&{-1.93E+00$\pm$3.84E-02}\\
&(0, {\bf 6.02E+02})&(0, 6.08E+02)&(0, 6.72E+02)&(0, 6.22E+02)&(0, 6.20E+02)&(0, 6.19E+02)&(0, 6.20E+02)\\
\hline
\multirow{2}{*}{$P_{8-1}$}&1.21E+05$\pm$4.61E+03&1.23E+05$\pm$2.70E+03&1.28E+05$\pm$1.14E+03&1.29E+05$\pm$2.99E+03&{\bf 1.30E+05$\pm$2.04E+02}&1.30E+05$\pm$2.89E+02&1.28E+05$\pm$2.66E+03\\
&(0, 7.35E-01)&(0, {\bf 6.85E-01})&(2, 3.28E+00)&(18, 9.82E-01)&({\bf 52}, 1.25E+00)&(20, 9.39E-01)&(10, 8.97E-01)\\
\hline
\multirow{2}{*}{$P_{8-2}$}&7.62E+03$\pm$4.80E+02&8.02E+03$\pm$1.19E+02&8.49E+03$\pm$4.21E+01&8.66E+03$\pm$3.56E+01&{\bf 8.72E+03$\pm$4.45E+00}&8.70E+03$\pm$1.47E+01&8.70E+03$\pm$1.62E+01\\
&(0, 2.71E+01)&(0, {\bf 2.61E+01})&(0, 1.20E+02)&(0, 3.65E+01)&({\bf 84}, 4.41E+01)&(4, 3.43E+01)&(4, 3.25E+01)\\
\hline
\multirow{2}{*}{$P_{8-3}$}&5.30E+03$\pm$2.12E+02&5.24E+03$\pm$1.83E+02&6.01E+03$\pm$1.19E+01&6.87E+03$\pm$7.85E+01&{\bf 6.95E+03$\pm$0.00E+00}&6.84E+03$\pm$7.11E+01&6.93E+03$\pm$3.66E+01\\
&(0, 4.29E+00)&(0, {\bf 4.13E+00})&(0, 1.92E+01)&(26, 5.88E+00)&({\bf 100}, 7.16E+00)&(2, 5.51E+00)&(58, 5.23E+00)\\
\hline
\multirow{2}{*}{$P_{8-4}$}&6.52E+03$\pm$4.14E+02&6.43E+03$\pm$2.22E+02&7.19E+03$\pm$1.89E+02&8.81E+03$\pm$1.02E+02&8.71E+03$\pm$1.06E+02&8.70E+03$\pm$9.21E+01&{\bf 8.87E+03$\pm$5.43E+01}\\
&(0, 6.04E+00)&(0, {\bf 5.90E+00})&(0, 2.75E+01)&({\bf 8}, 8.31E+00)&(0, 1.01E+01)&(0, 7.73E+00)&(4, 7.28E+00)\\
\hline
\multirow{2}{*}{$P_{8-5}$}&8.10E+03$\pm$5.96E+02&8.37E+03$\pm$2.87E+02&9.33E+03$\pm$2.29E+02&1.11E+04$\pm$4.40E+01&1.09E+04$\pm$7.01E+01&1.10E+04$\pm$8.22E+01&{\bf 1.12E+04$\pm$1.86E+01}\\
&(0, 7.09E+00)&(0, {\bf 6.91E+00})&(0, 3.28E+01)&(2, 9.64E+00)&(0, 1.17E+01)&(0, 8.87E+00)&({\bf 6}, 8.29E+00)\\
\hline
\end{tabular}
\end{sidewaystable}

\begin{table}[htbp]
\renewcommand{\arraystretch}{1.0}
\caption{Ranks on the performances of  the compared algorithms for the selected benchmark problems.}\label{tab_4}
\centering
\begin{tabular}{cccccccc}
\hline\hline
Problem.&AMPSO&AMDE&DisABC&BPSO&binDE&AQDE&BLDE\\
\hline
$p_1$&1&1&1&1&6&7&1\\
$p_2$&1&3&4&6&5&7&2\\
$p_3$&7&3&6&2&4&5&1\\
$p_4$&5&2&6&7&4&3&1\\
$p_5$&4&3&7&1&6&5&2\\
$p_6$&5&2&7&3&6&4&1\\
$p_7$&7&3&6&1&5&4&2\\
$p_{8-1}$&7&6&5&3&1&2&4\\
$p_{8-2}$&7&6&5&4&1&2&3\\
$p_{8-3}$&6&7&5&3&1&4&2\\
$p_{8-4}$&6&7&5&2&3&4&1\\
$p_{8-5}$&7&6&5&2&4&3&1\\
$Average$&5.3&4.1&5.2&2.9&3.8&4.2&1.8\\
 \hline

\end{tabular}

\end{table}

Numerical results in Tab. \ref{tab_3} show that BLDE is generally competitive to the compared algorithms for the selected benchmark problems, which is also illustrated by Tab. \ref{tab_4}, where BLDE averagely ranks first for the benchmark problems. Meanwhile, because it contains no time-consuming operations, for most cases BLDE spends less CPU time for the selected benchmark problems. Considering that AveFit and StdDev are two overall statistical indexes of the numerical results, we also perform a Wilcoxon rank sum test \cite{Gibbons2003} with a significance level of 0.05 to compare performances of the tested algorithms, and the results are listed in Tab. \ref{tab_5}.

 \begin{sidewaystable}[htbp]
\renewcommand{\arraystretch}{1.0}
\caption{Wilcoxon rank sum tests of the compared algorithms on the benchmark problems. The notation $+(-)$ means the algorithm for comparison is significantly superior to (inferior to) BLDE with significance level 0.05; $\approx$ means the compared algorithm is not significantly different with BLDE.}\label{tab_5}
\centering
\begin{tabular}[\textwidth]{clcl}
  \hline\hline
Algorithm & HBPD & Algorithm & HBPD\\
 \hline
\multirow{3}{*}{AMPSO}&$+:\emptyset$&\multirow{3}{*}{AMDE}&$+:\emptyset$\\
&$\approx:P_{1},P_{2}$&&$\approx:P_{1},P_{2},P_4,P_5$\\
&$-:P_{3},P_{4},P_{5},P_{6},P_{7},P_{8-1},P_{8-2},P_{8-3},P_{8-4},P_{8-5}$&&$-:P_{3},P_{6},P_{7},P_{8-1},P_{8-2},P_{8-3},P_{8-4},P_{8-5}$\\
\\
\multirow{3}{*}{DisABC}&$+:\emptyset$&\multirow{3}{*}{BPSO}&$+:P_5,P_7$\\
&$\approx:P_{1},P_{8-1},$&&$\approx:P_{1},P_{8-1}$\\
&$-:P_{2},P_{3},P_{4},P_{5},P_{6},P_{7},P_{8-2},P_{8-3},P_{8-4},P_{8-5}$&&$-:P_{2},P_{3},P_{4},P_{6},P_{8-2},P_{8-3},P_{8-4},P_{8-5}$\\
\\
\multirow{3}{*}{binDE}&$+:P_{8-1},P_{8-2},P_{8-3}$&\multirow{3}{*}{AQDE}&$+:P_{8-1}$\\
&$\approx:P_{1}$&&$\approx:P_{8-2}$\\
&$-:P_{2},P_{3},P_{4},P_{5},P_{6},P_{7},P_{8-4},P_{8-5}$&&
$-:P_1,P_{2},P_{3},P_{4},P_{5},P_{6},P_{7},,P_{8-3},P_{8-4},P_{8-5}$\\
\hline
\end{tabular}

\end{sidewaystable}

The results of Wilcxon rank sum tests demonstrate that BPSO performs significantly better on $P_5$ and $P_7$, the nosiy quadric problem and the maximization problem of Ackley's function, respectively. Because BPSO imitates the evolving mechanisms of PSO by simultaneously changing all bits of the individuals, it can quickly converge to the global optimal solutions. However, BLDE sometimes mutates bit by bit, and consequently, its evolving process is more vulnerable to be influenced by noises and the multimodal landscapes of benchmark problems. Thus, BPSO also performs better than BLDE on $P_5$ and $P_7$. For similar reasons, BPSO outperform BLDE on $P_{8-1}$, a low-dimensional MKP.

Meanwhile, binDE obtains better results than BLDE on the low-dimensional MKPs $P_{8-1}-P_{8-3}$, but performs worse than BLDE on the other problems, which is attributed to the fact that the exploitation ability of binDE descend with the expansion of the searching space. Consequently, binDE cannot perform well on the high-dimensional problems. Similarly, AQDE, which is specially designed for Knapsack problems, only outperforms BLDE for the low-dimensional MKP $P_{8-1}$, and cannot perform better than BLDE for other selected benchmark problems.

\subsection{Further comparison on the exploration and exploitation abilities }

To further explore the underlying causes resulting in BLDE performing worse than the BPSO, binDE and AQDE on several given test problems, we try to investigate how their exploration and exploitation abilities change during the evolving process. Thus, a renewal metric and a refinement metric are defined to respectively quantify the exploration and exploitation abilities.

\begin{definition}
\label{def_3}
Denote the population of an EA at the $t^{th}$ generation to be $\mathbf X^{(t)}$, which consists of $\mu$ n-bit individuals. Let $HammDist(\mathbf x,\mathbf y)$ to be the Hamming distance between two binary vectors $\mathbf x$ and $\mathbf y$. The \textbf{renewal metric} of an EA at the $t^{th}$ generation is defined as
\begin{equation}
\label{plor}
\alpha(t)\triangleq \frac{1}{\mu\cdot n}\sum_{i=1}^{\mu}Ham(\mathbf x^{(t)}(i)-\mathbf {tx}^{(t)}(i)),
\end{equation}
where $\mathbf x^{(t)}(i)$ is the $i^{th}$ individual in $\mathbf X^{(t)}$, and $\mathbf {tx}^{(t)}(i)$ is the corresponding candidate solution.
The \textbf{refinement metric} of an EA at the $t^{th}$ generation is defined as
\begin{equation}
\label{ploi}
\beta(t)\triangleq \frac{1}{\mu\cdot n}\sum_{i=1}^{\mu}\left(n-Ham(\mathbf x^{(t)}(i)-\mathbf x_{gb}(t))\right),
\end{equation}
where $\mathbf x_{gb}(t)$ is the best explored solution before the $t^{th}$ generation.
\end{definition}

The Hamming distance between $\mathbf x^{(t)}(i)$ and the corresponding trial vector $\mathbf{tx}(i)$ denotes the the overall changes that is performed on the bit-string by the variation strategies. Accordingly, the average value over the whole population can indicate the overal changes of the population. Then, $\alpha(t)$ properly reveals the exploration abilities of EAs at generation $t$. Meanwhile, an EA which harbors a big value of $\beta(t)$ can intensely exploit the local region around the best explored solution $\mathbf x_{gb}$, and thus, it harbors powerful exploitation ability.

For the comparison, we illustrate the changing curves of the renewal metric and the refinement metric for BLDE, BPSO,  AQDE and binDE by Figure \ref{fig_1}. Fig.\ref{Fig.sub.1} and Fig.\ref{Fig.sub.2} show that when BPSO is employed to solve $P5$ and $P7$, the renewal metric quickly descend to about zero, and the refinement metric ascend to a high level, which demonstrates that the population of PSO quickly converges. Meanwhile, the diversity of the population rapidly descend to a low level, and the population focuses on local search around the obtained best solution. Since the intensity of noise in $P_5$ is small, the convergence of BPSO is not significantly influenced. For $P_7$, the massive local optimal solutions are regularly distributed in the feasible region, BPSO can also quickly locate the global optimal solution. However, BLDE tries to keep a balance between exploration and exploitation, and the bit-by-bit variation strategies make it more vulnerable to be frustrated by the noise of $P_5$ as well as the multi-modal landscape of $P_7$. As a consequence, BPSO performs better than BLDE on $P_5$ and $P_7$.

However, the local optimal solutions of MKPs are not regularly distributed. Thus, to efficiently explore the feasible regions, it is vital to keep a balance between exploration and exploitation. Figs. \ref{Fig.sub.3}, \ref{Fig.sub.4}, \ref{Fig.sub.5} and \ref{Fig.sub.6} demonstrate binDE and AQDE can keep a balance between exploration and exploitation for the compared algorithms. Thus, AQDE performs better than BLDE on the test problem $P_{8-1}$, and binDE performs better than BLDE on  $P_{8-1}$, $P_{8-2}$ and $P_{8-3}$.

\begin{figure}[htbp]
\centering
\subfigure[$P_5$ : BLDE vs. BPSO;]{
\label{Fig.sub.1}
\includegraphics[width=3.0in]{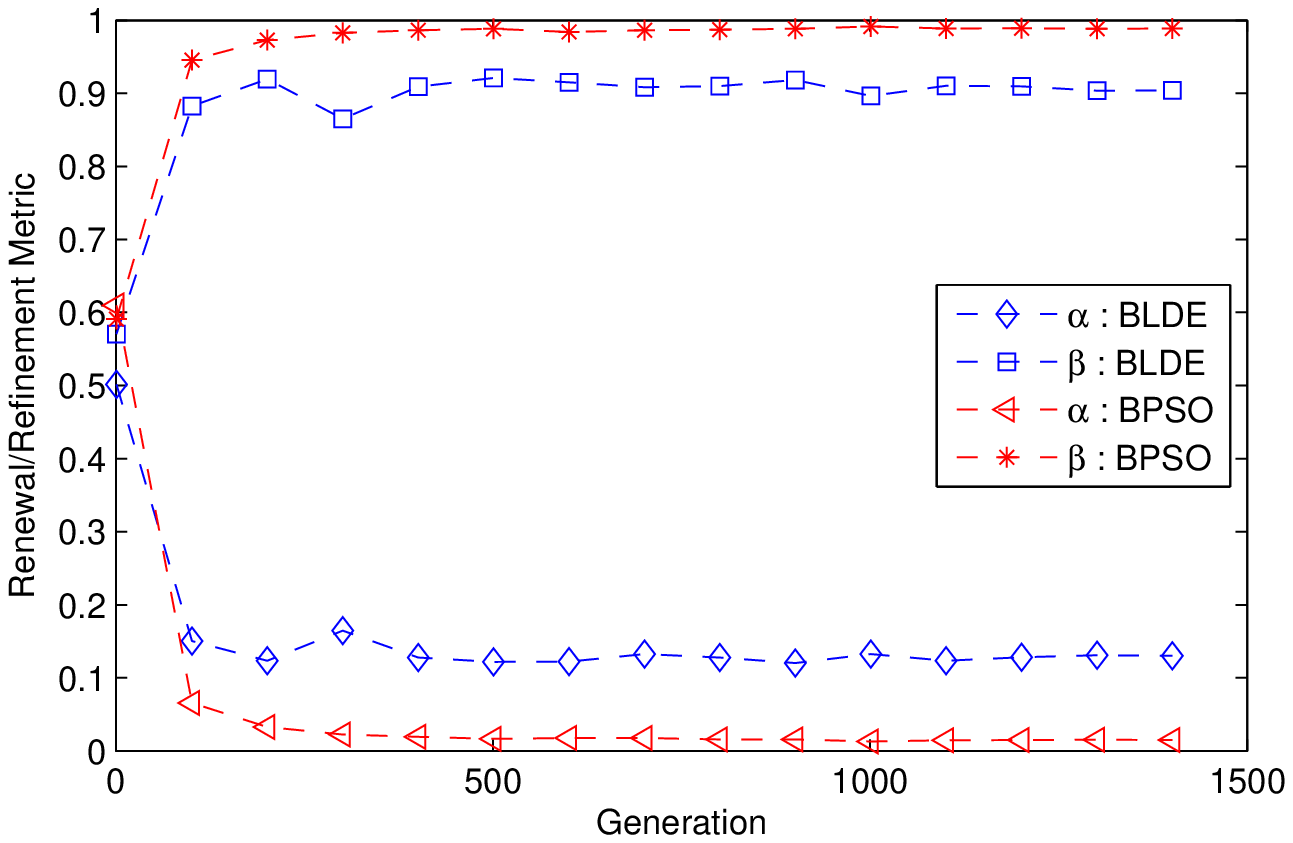}}
\subfigure[$P_7$: BLDE vs. BPSO;]{
\label{Fig.sub.2}
\includegraphics[width=3.0in]{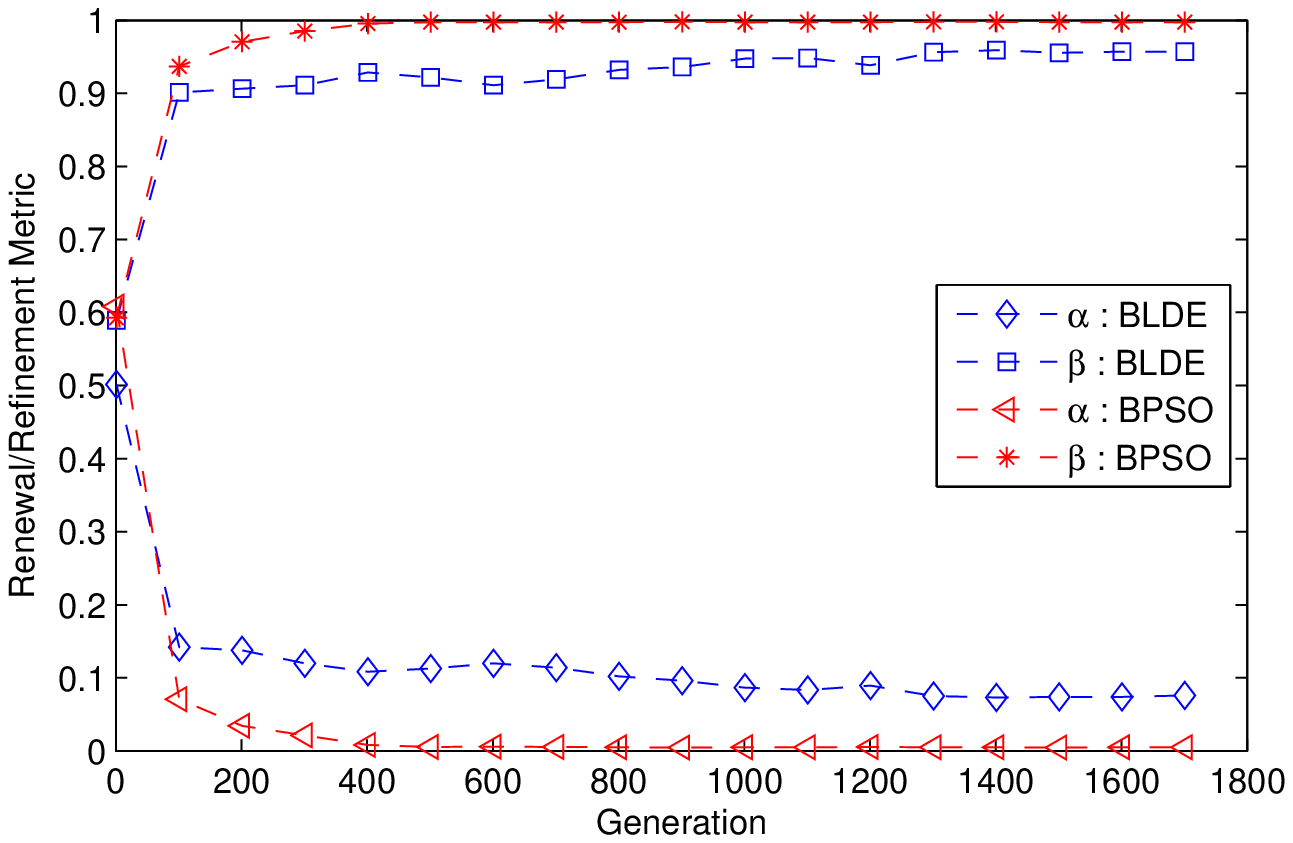}}\\

\subfigure[$P_{8-1}$: BLDE vs. AQDE;]{
\label{Fig.sub.3}
\includegraphics[width=3.0in]{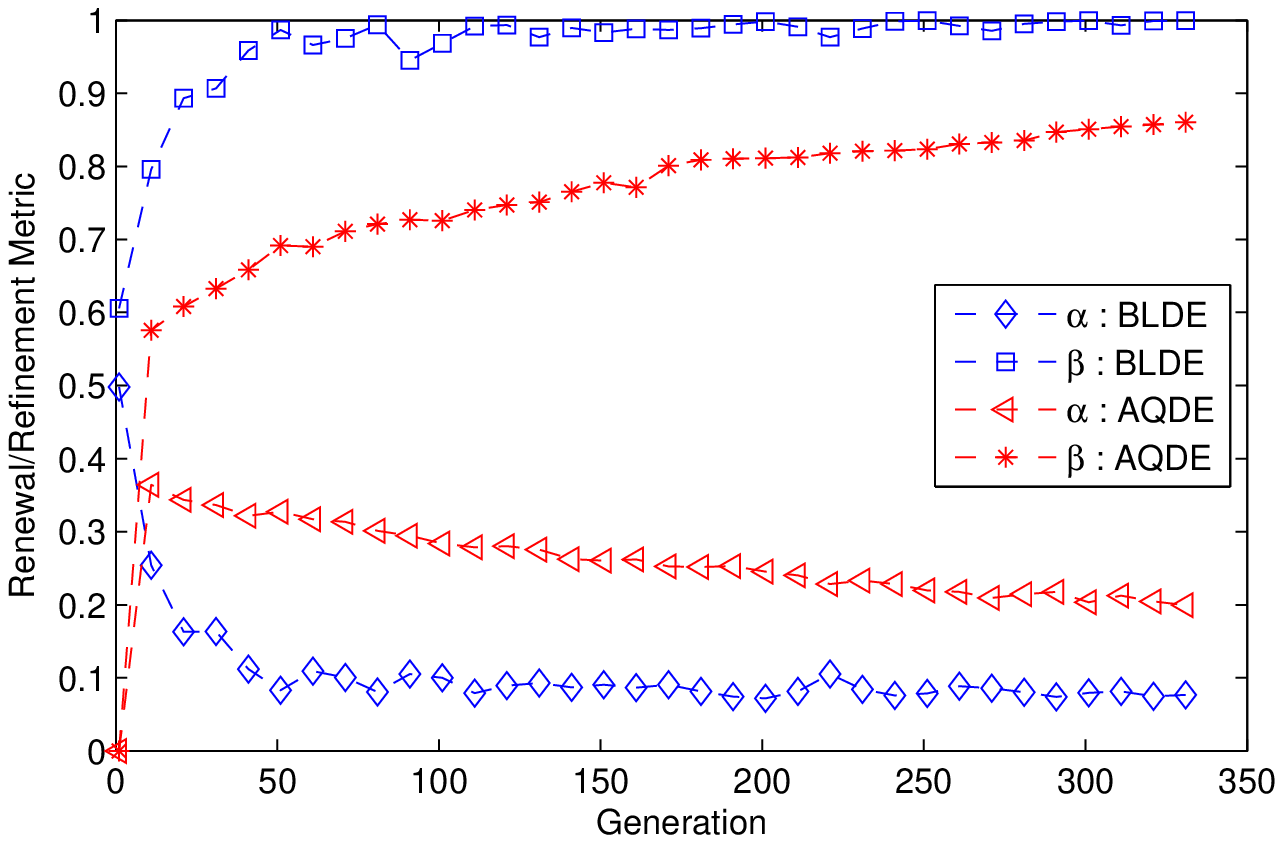}}
\subfigure[$P_{8-1}$: BLDE vs. binDE;]{
\label{Fig.sub.4}
\includegraphics[width=3.0in]{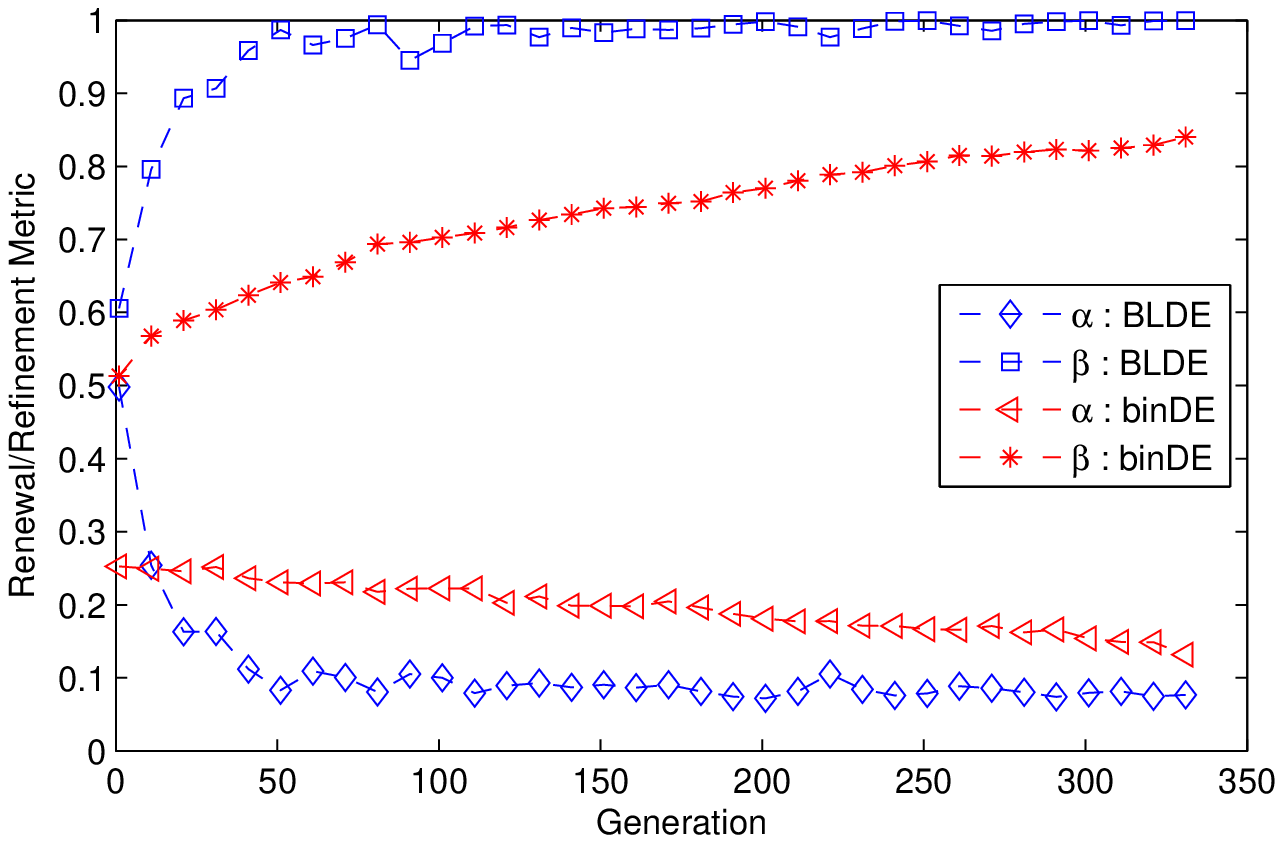}}\\

\subfigure[$P_{8-2}$: BLDE vs. binDE;]{
\label{Fig.sub.5}
\includegraphics[width=3.0in]{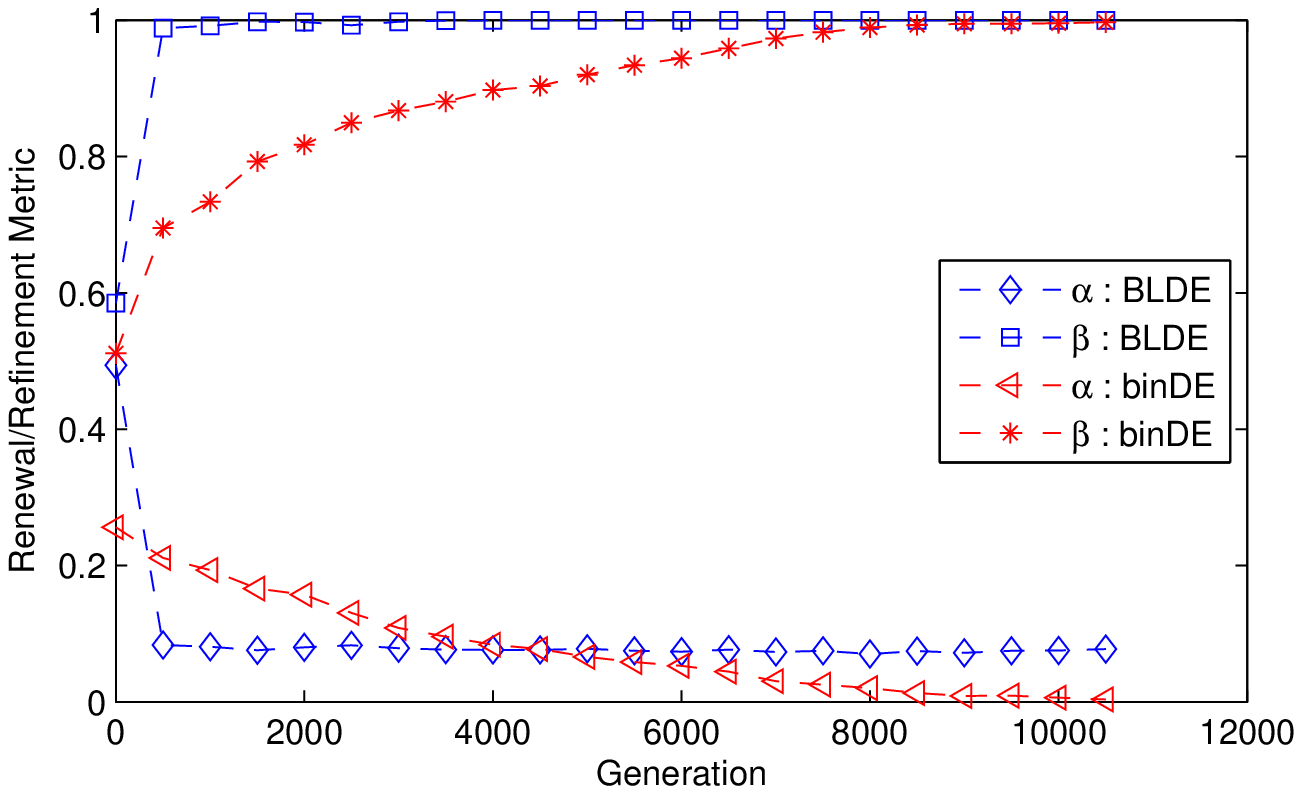}}
\subfigure[$P_{8-3}$: BLDE vs binDE.]{
\label{Fig.sub.6}
\includegraphics[width=3.0in]{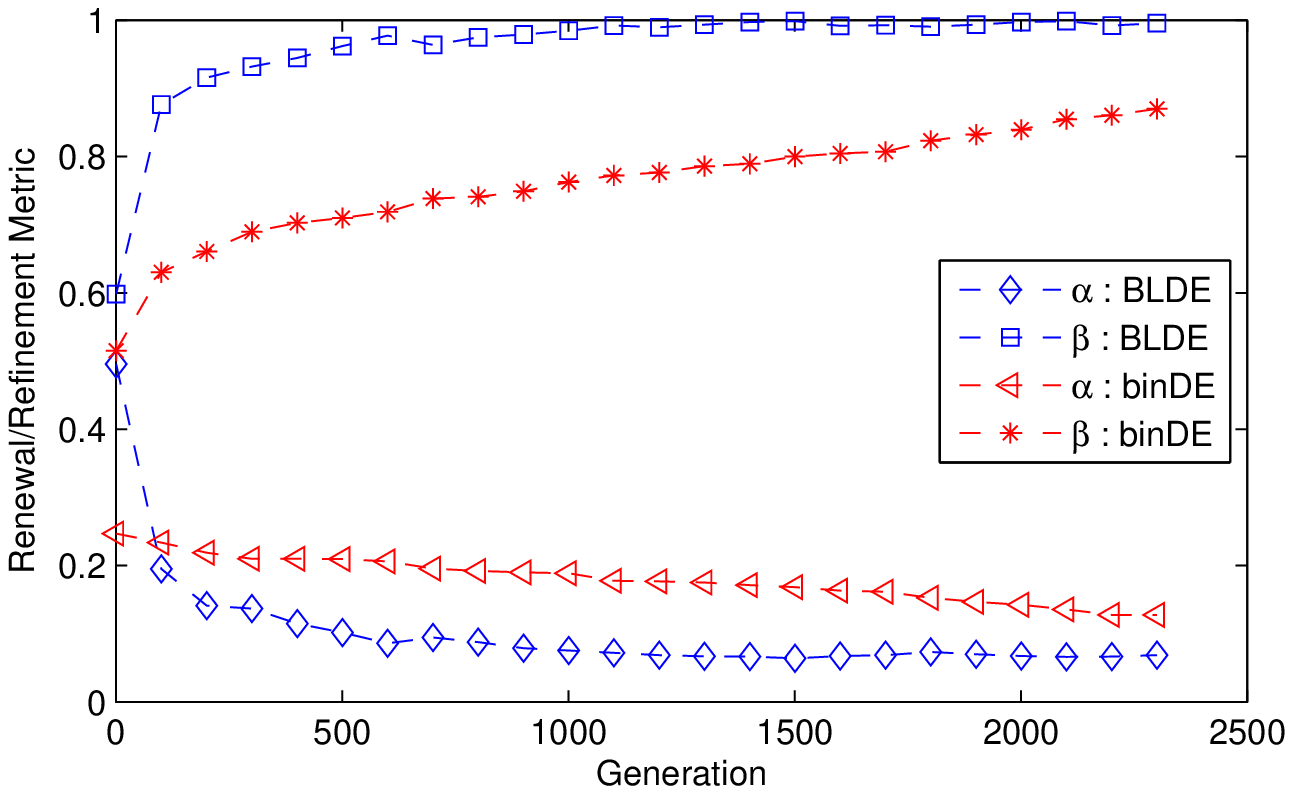}}\\
\caption{Comparisons of the renewal and refinement metrics for test problems $P_5$, $P_7$, $P_{8-1}$, $P_{8-2}$, $P_{8-3}$.}
\label{fig_1}
\end{figure}

\section{Performance of BLDE on the unit commitment problem}
In this section, we employ BLDE solving the unit commitment problem (UCP) in power systems. To minimize the production cost over a daily to weekly time horizon, UCP involves the optimum scheduling of power generating units as well as the determination of the optimum amounts of power to be generated by committed units\footnote{To compare with the work reported in \cite{Datta2012}, we employ similar notations and descriptions in this section.} \cite{Datta2012}. Thus, UCP is a mixed integer optimization problem, the decision variables of which are the binary string representing the on/off statuses of units and the real variables indicating the generated power of units.

\subsection{Objective function of UCP}
The objective of UCP is to minimize the total production cost
\begin{equation}
 F=\sum_{t=1}^T\sum_{i=1}^N\left[\phi_i(P_{it})\cdot u_{it}+\psi_{it}\cdot (1-u_{i,t-1})\cdot u_{i,t}\right]
\end{equation}
where $N$ is the number of units to be scheduled and T is the time horizon. When the $t^{th}$ unit is committed to generate power $P_{it}$ at time $t$, the binary variable $u_{it}$ is set to be 1; otherwise, $u_{it}=0$. The function $\phi_i(P_{it})$ represents the fuel cost of unit $i$ at time $t$, which is frequently approximated by
\begin{equation}
\phi_i(P_{it})=a_i+b_ip_{it}+c_iP^2_{it}
\end{equation}
where $a_i$, $b_i$ and $c_i$ are known coefficients of unit $i$. If the $i^{th}$ unit has been off prior to start-up, there is a start-off cost \begin{equation}
\psi_{it}=\left\{
\begin{aligned}& d_i,\quad\mbox{if}\,\,\Gamma_i^{down}\le\tau_{it}^{off}\le\Gamma_i^{down}+f_i\\
& e_i,\quad\mbox{if}\,\,\tau_{it}^{off}>\Gamma_i^{down}+f_i\end{aligned}
\right.
\end{equation}
where $d_i$, $e_i$, $f_i$ and $\Gamma_i^{down}$ are the hot start cost, cold start cost, cold start time and minimum down time of unit $i$, respectively. $\tau_{it}^{off}$, the continuously off time of unit $i$, is determined by
\begin{equation}
\tau_{it}^{off}=\left\{ \begin{aligned}&0,& \mbox{if}&\,\,u_{it}=1\\
&1,& \mbox{if}&\,\,u_{it}=0,\,t=1 \mbox{ and } \sigma_i>0\\
&1-\sigma_i,& \mbox{if}& \,\,u_{it}=0,\,t=1\mbox{ and } \sigma_i<0\\
&1+\tau_{i,t-1}^{off},& \mbox{if} &\,\,u_{it}=1 \mbox{ and } t>1\\
 \end{aligned}
\right.
\end{equation}
where $\sigma_i$ is the initial status of unit $i$, which shows for how long the unit was on/off prior to the start of the time horizon.
\subsection{Constraints in UCP}
The minimization of the total production cost is subject to the following constraints.
\begin{description}
\item[Power balance constraints:] The total generated at time $t$ must meet the power demand at that time instant, i.e.,
    \begin{equation}
    \sum_{i=1}^Nu_{it}P_{it}=D_t,\quad t=1,2,\dots,T
    \end{equation}
    where $D_t$ is  the power demand at time $t$. Practically, it is hardly possible to exactly meet the power demand, an error $\epsilon$ is allowed for the generated power, i.e.,
    \begin{equation}
    \left|\dfrac{\sum_{i=1}^Nu_{it}P_{it}}{D_t}-1\right|\le\epsilon,\quad t=1,2,\dots,T.
    \end{equation}
\item[Spinning reserve constraints:] Due to possible outages of equipments, it is necessary for power systems to satisfy the spinning reserve constraints. Thus, the sum of the maximum power generating capacities of all committed units should be greater than or equal to the power demand plus the minimum spinning reserve requirement, i.e.,
    \begin{equation}
    \sum_{i=1}^Nu_{it}P_{i}^{max}\ge D_t+R_t,\quad t=1,2,\dots,T
    \end{equation}
    where $P_i^{max}$ is the maximum power generating capacity of unit $i$, and $R_t$ is the minimum spinning reserve requirement at time $t$.
\item[Minimum up time constraints:] If unit $i$ is on at time $t$ and switched off at time $t+1$, the continuous up time $\tau_{it}^{on}$ should be greater than or equal to the minimum up time $\Gamma_i^{up}$ of unit $i$, i.e.,
    \begin{equation}
    \tau_{it}^{on}\ge\Gamma_i^{up},\quad\mbox{if } u_{it}=1, u_{i,t+1}=0 \mbox{ and }t<T,\quad i=1,\dots,N
    \end{equation}
    where the continuously up time is
    \begin{equation}
\tau_{it}^{on}=\left\{ \begin{aligned}&0,& \mbox{if}&\,\,u_{it}=0\\
&1,& \mbox{if}&\,\,u_{it}=1,\,t=1 \mbox{ and } \sigma_i<0\\
&1+\sigma_i,& \mbox{if}& \,\,u_{it}=1,\,t=1\mbox{ and } \sigma_i>0\\
&1+\tau_{i,t-1}^{on},& \mbox{if} &\,\,u_{it}=1 \mbox{ and } t>1.\\
\end{aligned}
\right.
\end{equation}
\item[Minimum down time constraints:]If unit $i$ is off at time $t$ and switched on at time $t+1$, the continuous up time $\tau_{it}^{off}$ should be greater than or equal to the minimum off time $\Gamma_i^{down}$ of unit $i$, i.e.,
    \begin{equation}
    \tau_{it}^{off}\ge\Gamma_i^{down},\quad\mbox{if } u_{it}=0, u_{i,t+1}=1 \mbox{ and }t<T,\quad i=1,\dots,N
    \end{equation}
\item[Range of generated power:] The generated power of a unit is limited in an interval, i.e.,
    \begin{equation}
    P_i^{min}\le P_{it}\le P_i^{max},\quad i=1,2,\dots,N \mbox{ and } t=1,2,\dots, T
    \end{equation}
    where $P_i^{min}$ and $P_i^{max}$ is the minimum power output and the maximum power output of unit $i$, respectively.
\end{description}

\subsection{Implement of BLDE for UCP}
The optimal commitment of power units in UCP is obtained by combining BLDE with real-coded DE operations. In BLDE, each binary individual represents an on/off scheduling plan of units, accompanied with a real-coded individual representing the specific power outputs of units. When the binary individuals are recombined during the iteration process, the real-coded individuals are recombined via the {\em DE/rand/1} mutation and binary crossover strategies of the real-coded DE. Then, binary individuals and the corresponding real individuals are integrated together for evaluation. If the combined mixed-integer individuals violate the constraints in UCP, they are repaired via the repairing mechanisms proposed in \cite{Datta2012}.

The performance of BLDE is tested via a 10-unit power system, the parameters and forecasted power demands of which are respectively listed in Tab. \ref{tab_6} and Tab. \ref{tab_7}. To fairly compare BLDE with the method proposed in \cite{Datta2012}, we also set the population size to be 100, and the results are compared after 30 independent runs of 2500 iterations, where the scalar factor $F$ is set to be 0.8. The statistical results are listed in Tab. \ref{tab_8}.
\begin{sidewaystable}[htbp]
\renewcommand{\arraystretch}{1.0}
\caption{Unit parameters for the 10-unit power system.}
\label{tab_6}
\centering
\begin{tabular}{cccccccccccc}
\hline\hline
Unit($i$)& $P_{i}^{max}(MW)$ &$P_{i}^{min}(MW)$& $a_i(\$/h)$ & $b_i(\$/MWh)$ & $c_i(\$/MW^2h)$&$d_i(\$)$ & $e_i(\$)$ &$f_i(h)$ & $\Gamma_i^{ip}(h)$ & $\Gamma_i^{down}(h)$ & $\sigma_i$(h)\\
\hline
1&455&150&1000&16.19&0.00048&4500&9000&5&8&8&8\\
2&455&150&970&17.26&0.00031&5000&10000&5&8&8&8\\
3&130&20&700&16.60&0.00200&550&1100&4&5&5&-5\\
4&130&20&680&16.50&0.00211&560&1120&4&5&5&-5\\
5&162&25&450&19.70&0.00398&900&1800&4&6&6&-6\\
6&80&20&370&22.26&0.00712&170&340&2&3&3&-3\\
7&85&25&480&27.74&0.00079&260&520&2&3&3&-3\\
8&55&10&660&25.92&0.00413&30&60&0&1&1&-1\\
9&55&10&665&27.27&0.00222&30&60&0&1&1&-1\\
10&55&10&670&27.79&0.00173&30&60&0&1&1&-1\\
\hline
\end{tabular}

\end{sidewaystable}

\begin{table}[htbp]
\renewcommand{\arraystretch}{1.0}
\caption{Forecasted power demands for the 10-unit system over 14-h time horizon.}
\label{tab_7}
\centering
\begin{tabular}{ccccccccccccc}
\hline
Hour& 1& 2 &3&4&5&6&7&8&9&10&11&12\\
Demand (MW)&700&700&850&950&1000&1100&1150&1200&1300&1400&1450&1500\\
&&&&&&&&&&&&\\
Hour&13&14&15&16&17&18&19&20&21&22&23&24\\
Demand (MW)&1400&1300&1200&1050&1000&1100&1200&1400&1300&1100&900&800\\
\hline
\end{tabular}

\end{table}

\begin{table}[htbp]
\renewcommand{\arraystretch}{1.0}
\caption{Results comparison between BLDE and BRCDE\cite{Datta2012} for the 10-unit power system.``-'' means that the corresponding item was not presented in the literature.}
\label{tab_8}
\centering
\begin{tabular}{cccccc}
\hline\hline
Method &Power balance error $\epsilon$& Best cost &Average Cost &Worst Cost & Standard deviation\\
\hline
{BRCDE}& 0.0\%&563938 &-&-&-\\
&0.1\%&563446&563514&563563&30\\
&0.5\%&561876&-&-&-\\
&1\%&559357&-&-&-\\
{BLDE}& 0.0\%&563977 &564005&564088&24\\
&0.1\%&563552&563636&563745&49\\
&0.5\%&561677&561847&-&50\\
&1\%&559155&559207&559426&48\\

\hline
\end{tabular}

\end{table}
The comparison results show that when the power balance error $\epsilon$ is small, performance of BLDE is a bit worse than that of the binary-real-coded differential evolution (BRCDE) algorithm proposed in \cite{Datta2012}. However, when the power balance is relaxed to a relatively great extent, BLDE outperform BRDE for UCP of the 10-unit power system. The reason could be that crossover operation for real variables is not appropriately regulated for UCP, and accordingly, simultaneous variations on all real variables usually lead to violations of constraints. Thus, BLDE can only outperforms BRCDE when the constraints are relaxed greatly.

\section{Discussions}
In this paper, we propose a BLDE algorithm appropriately incorporating the mutation strategy of binary DE and the learning mechanism of binary PSO. For majority of the selected benchmark problems, BLDE can outperform the compared algorithms, which indicate that BLDE is competitive to the compared algorithms. However, statistical test results show that BPSO performs better than BLDE on $P_5$ and $P_7$, AQDE is more efficient for $P_{8-1}$, and binDE obtains better results on $P_{8-1}$, $P_{8-12}$ as well as $P_{8-3}$. When generating a candidate solution, BLDE first initiate it as the winner of two obtained solutions, and then, regulate it by learning from the best individual in the population. This strategy simultaneously incorporates the synchronously changing strategy and the bitwise mutation strategy of candidate generation. Thus, BLDE can performs well on most of the high-dimensional benchmark problems. However, when BLDE is employed to solve $P_5$ and $P_7$, the global optimal solutions of which are easy to be locate, it performs worse than BPSO; meanwhile, when it is implemented to solve the low-dimensional problems $P_{8-1}$, $P_{8-2}$ and $P_{8-3}$, the local optimal solutions of which are irregularly distributed in the feasible regions, it cannot perform better than binDE.

\section{Conclusions}
Generally, the proposed BLDE is competitive to the existing binary evolutionary algorithms. However, its performance can  been improved. Thus, future work will focus on designing an adaptive strategy appropriately managing the synchronously changing strategy and the bitwise mutation strategy employed in BLDE. Meanwhile, we will try to further improve its performances on mixed-integer optimization problems by efficiently incorporate it with real-coded recombination strategies.

\section*{Acknowledgements}
This work was partially supported by the Natural Science Foundation of China under Grants 51039005, 61173060 and 61303028, as well as the Fundamental Research Funds for the Central Universities (WUT: 2013-Ia-001).


\begin{thebibliography}{00}
\bibitem{Banks2007}
Banks A., Vincent J. and Anyakoha C., A review of particle swarm optimization, Part I: background and development. Natural Computing, 6(4): 467-484, 2007.

\bibitem{Banks2008}
Banks A., Vincent J. and Anyakoha C., A review of particle swarm optimization, Part II: hybridisation, combinatorial, multicriterial and constrained optimization, and indicative applications. Natural Computing, 7(1): 109-124, 2008.

\bibitem{Das2005}
Das S., Konar A. and Chakraborty U. K., Improving particle swarm optimization with differentially perturbed velocity. In {\em Proc. of 2005 Conference on Genetic and Evolutionary Computation (GECCO'05)}, ACM, 2005, pp. 177-184.

\bibitem{Das2011}
Das S. and Suganthan P. N., Differential evolution: a survey of the state-of-the-art. IEEE Transactions on Evolutionary Computation, 15(1):4-31, 2011.

\bibitem{Datta2012}
Datta D. and Dutta S., A binary-real-coded differential evolution for unit commitment problem.  Electrical Power and Energy Systems, 42(1): 517-524, 2012.

\bibitem{Eberhart2001}
Eberhart R. C. and Shi Y., Particle swarm optimizaiton: developments, applications and resources. In {\em Proc. of 2001 IEEE Conference on Evolutionary Computation (CEC 2001)}, IEEE, 2001, pp.81-86.

\bibitem{Engelbrecht2007}
Engelbrecht A. P. and Pampar\'{a}, Binary differential evolution strategies. In {\em Proc. of 2007 IEEE Congress on Evolutionary Computation}, IEEE, 2007, pp.1942-1947.

\bibitem{Gibbons2003}
Gibbons, J. and Chakraborti, S., Nonparametric Statistical Inference (the Fifth Edition). Taylor and Francis, 2011.

\bibitem{Gong2007}
Gong, T. and Tuson, A.L., Differential evolution for binary encoding.
 In {\it Soft Computing in Industrial Applications}, Springer, 2007, pp. 251-262.

\bibitem{He2007}
He X. and Han L., A novel binary differential evolution algorithm based on artificial immune system. In {\em Proc. of 2007 IEEE Congress on Evolutionary Computation}, IEEE, 2007, pp. 2267-2272.

\bibitem{Horn1994}
Horn J., Goldber D. E. and Deb K., Long path problems. In {\em Proc. of 1994 International Conference on Parallel Problem Solving from Nature (PPSN III)}, Springer, 1994, pp. 149-158.

\bibitem{Hota2010}
Hota  A. R. and Pat, A., An adaptive quantum-inspired differential evolution algorithm for 0-1 knapsack problem. In {\em Proc. of 2010 Second World Congress on Nature and Biologically Inspired Computing (NaBIC)}, IEEE, 2010, pp.703-708.

\bibitem{Kashan2012}
Kashan, M.H., Nahavandi, N., \& Kashan, A.H. (2012). DisABC: A new artificial bee colony algorithm for binary optimization.  Applied Soft Computing, 12(1): 342-352, 2012.

\bibitem{Kashan2013}
Kashan M. H., Kashan A. H. and Nahavandi N., A novel differential evolution algorithm for binary optimizatoin. Computational Optimization and Applications, 55(2): 481-513, 2013.

\bibitem{Kennedy1995}
Kennedy, J., \& Eberhart, R.C. (1995). Particle swarm optimization.
In {\em Proc. of 1995 IEEE International Conference on Neural Networks}, IEEE, 1995, pp. 1942-1948.

\bibitem{Kennedy1997}
Kennedy, J., \& Eberhart, R.C. (1997). A  discrete  binary  version  of  the particle  swarm  algorithm.
In {\em Proc. of  IEEE  International  Conference  on Systems, Man, and Cybernetics}, IEEE, 1997, pp. 4104-4108.

\bibitem{Kennedy2010}
Kennedy J., Particle swarm optimizaiton. In {\em Encyclopedia of Machine Learning}, Springer, 2010, pp.760-766.

\bibitem{Moore2006}
Moore P. W. and Venayagamoorthy G. K., Evolving digital circuit using hybrid particle swarm optimizaiton and differential evolution. Interational Journal of Neural Systems, 16(3): 163-177, 2006.

\bibitem{Omran2009}
Omran M. G. H., Engelbrecht A. P. and Salman A., Bare bones differential evolution. European Journal of Operational Research, 196(1): 128-139, 2009.

\bibitem{Pampara2005}
Pampara, G., Franken, N.,\& Engelbrecht, A.P., Combining particle swarm optimisation with angle modulation to solve binary problems. In {\em Proc. of 2005 IEEE Congress on Evolutionary Computation}, IEEE, 2005, pp. 89-96.

\bibitem{Pampara2006}
Pampara G., Engelbrecht A. P. and Franken, N., Binary differential evolution.
In {\em Proc. of 2006 IEEE Congress on Evolutionary Computation}, IEEE, 2006, pp. 1873-1879.

\bibitem{Ponsich2011}
Ponsich A. and Coello Coello C. A., Differential evolution performances for the solution of mixed-integer constraned process engineering problems. Appliced Soft Computing, 11: 399-409, 2011.

\bibitem{Poli2007}
Poli, R., Kennedy, J., and Blackwell, T., Particle swarm optimization: An Overview. Swarm Intelligence, 1: 33-57, 2007.

\bibitem{Price2005}
Price, K., Storn, R., and Lampinen, J., Differential Evolution: A Practical Approach
to Global Optimization. Springer, 2005.

\bibitem{Rekanos2008}
Rekanos I. T., Shape reconstruction of a perfectly conducting scatterer using differential evolution and particle swarm optimization. IEEE Transactions on Geoscience and Remote Sensing, 46(7): 1967-1974, 2008.

\bibitem{Storn1997}
Storn, R. and Price, K. (1997). Differential evolution-a simple and efficient adaptive scheme
for global optimization over continuous spaces.  Journal of Global Optimization, 11: 341-359, 1997.

\bibitem{Uyar2005}
Uyar, \c{S}., \& Eryi\u{g}it, G., Improvements to penalty-based evolutionary algorithms for the multi-dimensional knapsack problem using a gene-based adaptive mutation approach. In {\em Proc. of 2005 Conference on Genetic and evolutionary computation (GECCO '05)}, ACM, 2005, pp. 1257-1264.

\bibitem{Vesterstrom2004}
Vesterstrom J. and Thomsen R., A comparative study of differential evolution, particle swarm optimization and evolutionary algorithms on numerical benchmark problems. In {\em Proc. of 2004 IEEE Conference on Evolutionary COmputation (CEC'04)}, IEEE, 2004, pp. 1980-1987.

\bibitem{Wang2012}
Wang L., Fu X. P., Mao Y. F., Menhas M. I. and Fei M. R.,  A novel modified binary differential evolution algorithm and its applications.  Neurocomputing, 98: 55-75, 2012.

\bibitem{Wu2010}
Wu C-Y. and Tseng K-Y., Topology optimization of structures using binary differential evolution. Structural and Multidisciplinary Optimization, 42: 939-953, 2010.

\bibitem{Website}
\url{http://www.zib.de/index.php?id=921&no_cache=1&L=0&cHash=fbd4ff9555f8714ac6238261e3963432&type=98}

\bibitem{Yang2008}
Yang Q., A comparative study of discrete differential evolution on binary constraint satisfaction problems. In {\em Proc. of 2008 IEEE Congress on Evolutionary Computation}. IEEE, 2008, pp. 330-335.
\end{thebibliography}
\end{document}